\documentclass[letterpaper]{article} 
\usepackage{aaai25}  
\usepackage{times}  
\usepackage{helvet}  
\usepackage{courier}  
\usepackage[hyphens]{url}  
\usepackage{graphicx} 
\usepackage{natbib}  
\usepackage{caption} 
\usepackage[linesnumbered,ruled,vlined]{algorithm2e}
\usepackage{algorithmic}
\usepackage{amssymb, amsmath, amsthm}
\usepackage{bm}
\usepackage{MnSymbol}
\usepackage[inline]{enumitem}

\usepackage{newfloat}
\usepackage{listings}
\usepackage{subcaption}

\newtheorem{theorem}{Theorem}

\newtheorem{corollary}[theorem]{Corollary}

\newtheorem{claim}[theorem]{Claim}

\usepackage{todonotes}
\newenvironment{psketch}{%
  \proof}{\endproof}
\DeclareCaptionStyle{ruled}{labelfont=normalfont,labelsep=colon,strut=off} 
\newcommand{\bt}[2]{\mathsf{bt}(#1,#2)}

\title{Goal-Driven Reasoning in DatalogMTL with Magic Sets}
\author {
    Shaoyu Wang\textsuperscript{\rm 1}\equalcontrib,
    Kaiyue Zhao\textsuperscript{\rm 1}\equalcontrib,
    Dongliang Wei\textsuperscript{\rm 1}\equalcontrib, \\
    Przemys{\l}aw Andrzej Wa{\l}\c{e}ga\textsuperscript{\rm 2,\rm 3}, 
    Dingmin Wang\textsuperscript{\rm 2},
    Hongming Cai\textsuperscript{\rm 1},
    Pan Hu\textsuperscript{\rm 1}\thanks{Corresponding author}
}
\affiliations {
    \textsuperscript{\rm 1}School of Electronic Information and Electrical Engineering, Shanghai Jiao Tong University, China\\
    \textsuperscript{\rm 2}Department of Computer Science, University of Oxford, UK\\
    \textsuperscript{\rm 3}School of Electronic Engineering and Computer Science, Queen Mary University of London, UK\\
    \{royalrasins, Horizon52183, weidongliang9, hmcai, pan.hu\}@sjtu.edu.cn, \{przemyslaw.walega, dingmin.wang\}@cs.ox.ac.uk
}

\usepackage{bibentry}

\begin{document}

\maketitle

\begin{abstract}
DatalogMTL is a powerful rule-based language for temporal reasoning. Due to its high expressive power and flexible modeling capabilities, it is suitable for a wide range of applications, including tasks from industrial and financial sectors. However, due to its high computational complexity, practical reasoning in DatalogMTL is highly challenging. To address this difficulty, we introduce a new reasoning method for DatalogMTL which exploits the magic sets technique---a rewriting approach developed for (non-temporal) Datalog to simulate top-down evaluation with bottom-up reasoning. We have implemented this approach and evaluated it on publicly available benchmarks, showing that the proposed approach significantly and consistently outperformed state-of-the-art reasoning techniques.
\end{abstract}
\section{Introduction}

DatalogMTL \cite{DBLP:journals/jair/BrandtKRXZ18} extends the well-known declarative logic programming language Datalog \cite{DBLP:journals/tkde/CeriGT89} with operators from metric temporal logic (MTL) \cite{DBLP:journals/rts/Koymans90}, allowing for complex temporal reasoning. 
It has a number of applications, including ontology-based query answering~\cite{DBLP:journals/jair/BrandtKRXZ18,guzel2018ontop}, stream
reasoning~\cite{walkega2023stream,walega2019reasoning}, and reasoning for the financial sector \cite{DBLP:conf/edbt/ColomboBCL23,DBLP:conf/edbt/NisslS22,mori2022neural}, among others. 

To illustrate capabilities of DatalogMTL, consider a scenario in which we want to reason about social media interactions.
The following DatalogMTL rule
describes participation of users in the circulation of a viral social media post:
\begin{align}\label{rule:ex}
    \boxplus_{[0, 2]} P(x) \xleftarrow{} I(x, y) \wedge P(y),
\end{align}
namely, it states that if at a  time point $t$ a user $x$ interacted with a user $y$ over the post (expressed as $I(x,y)$), and $y$ participated in the post's circulation (expressed as $P(y)$), then in the time interval $[t, t+2]$, user $x$ will be continuously participating in the post's circulation ($\boxplus_{[0, 2]} P(x)$).

One of the main reasoning tasks considered in DatalogMTL is fact entailment, which involves checking whether a program-dataset pair $(\Pi, \mathcal{D})$ logically entails a given temporal fact. 
This task was shown to be \textsc{ExpSpace}-complete in combined complexity~\cite{DBLP:journals/jair/BrandtKRXZ18} and \textsc{PSpace}-complete in data complexity~\cite{DBLP:conf/ijcai/WalegaGKK19}. 
To decrease computational complexity, various syntactical~\cite{DBLP:conf/ijcai/WalegaGKK20} and semantic~\cite{DBLP:conf/kr/WalegaGKK20} modifications of DatalogMTL have been introduced.
DatalogMTL was also extended with stratified negation~\cite{DBLP:conf/aaai/CucalaWGK21}, non-stratified negation~\cite{DBLP:conf/kr/WalegaCKG21,walkega2024stable}, and  with temporal aggregation \cite{DBLP:conf/ruleml/BellomariniNS21}. 

A sound and complete reasoning approach for DatalogMTL can be obtained using an automata-based construction \cite{DBLP:conf/ijcai/WalegaGKK19}. 
Another approach for reasoning is to apply materialisation, that is,  successively apply to the initial dataset $\mathcal{D}$ rules in a given program $\Pi$, to derive all the facts entailed by a program-dataset pair $(\Pi, \mathcal{D})$.
It is  implemented, for example, within Vadalog reasoner~\cite{DBLP:conf/ruleml/BellomariniBNS22}. 
Materialisation, however, does not  guarantee termination since DatalogMTL programs can express (infinite) recursion over time. 
To marry completeness and efficiency of reasoning, a combination of automata- and materialisation-based  approaches was introduced in the MeTeoR system \cite{DBLP:conf/aaai/WangHWG22}. 
Materialisation-based approach typically outperforms the automata-based approach, so
cases in which  materialisation-based reasoning is sufficient were studied \cite{walkega2021finitely,walkega2023finite}.
More recently, a very efficient reasoning approach was introduced, which combines materialisation with detection of repeating periods \cite{DBLP:conf/aaai/WalegaZWG23}.

Despite recent advancement in the area, the currently available reasoning methods for DatalogMTL  \cite{DBLP:conf/aaai/WangHWG22, DBLP:conf/aaai/WalegaZWG23,DBLP:conf/ruleml/BellomariniBNS22} can suffer from deriving huge amounts of temporal facts which are irrelevant for a particular query. 
This, in turn, can make the set of derived facts too large to store in memory or even in a secondary storage. 
For example, 
if the input dataset mentions millions of users, together with  interactions between them and their participation in posts' circulation, application of Rule~\eqref{rule:ex} may result in a huge set of derived facts.
However, if we only want to check whether the user Arthur participated in the post's circulation on September 10th, 2023, most of the derived facts are irrelevant.
Note that determining which exactly facts  are relevant is not easy, especially when a DatalogMTL program contains a number of rules with complex interdependencies.

To address these challenges, we take inspiration from the \emph{magic set} rewriting technique, which was developed for non-temporal Datalog \cite{DBLP:books/aw/AbiteboulHV95, DBLP:books/cs/Ullman89,DBLP:journals/jcss/FaberGL07,magicPerf}.
We extend it to the temporal setting of DatalogMTL, thereby proposing a new variant of magic set rewriting. 
For example, consider again Rule~\eqref{rule:ex} and assume that we want to check if Arthur participated in the post's circulation on $10/9/2023$.
Observe that if there is no chain of interactions between Arthur and another user Beatrice in the input dataset, then  facts about Beatrice are irrelevant for our query, and so, we do not need to derive them.
We will achieve it by
rewriting the original program-dataset pair into a new pair (using additional magic predicates)
such that the new pair provides the same answers to the query as the original pair, but 
allows us for a more efficient reasoning.
As a result we obtain a goal-driven (i.e. query-driven) reasoning approach. To the best of our knowledge, our work is the first that allows to prune irrelevant derivations in DatalogMTL reasoning.

To implement the magic sets technique in the DatalogMTL setting, two main challenges need to be addressed. 
First, it is unclear how MTL operators affect sideways information passing in magic sets. 
We address this challenge by carefully examining each type of MTL operators and designing corresponding transformations. 
Second, for the transformed program-dataset pair established optimisations for reasoning in DatalogMTL may no longer apply or become less efficient. 
This is indeed the case, and so, we show how these optimisations can be adapted to our new setting.
We implemented our approach and tested it using the state-of-the-art DatalogMTL solver MeTeoR. 
Experiments show that compared with existing methods, our approach achieves a significant  performance improvement.

\section{Preliminaries}
We briefly summarise the core concepts of DatalogMTL and provide a recapitulation of magic set rewriting in the non-temporal setting. 

\subsubsection{Syntax of DatalogMTL.}
Throughout this paper, we assume that
time is continuous, in particular, that the
timeline is composed of rational numbers. 
A time interval is a set of continuous time points $\varrho$ of the form $\langle t_1, t_2\rangle$, where $t_1,t_2\in \mathbb{Q}  \cup  \{-\infty, \infty\}$, whereas 
$\langle$ is $[$ or $($ and likewise $\rangle$ is $]$ or $)$.
An interval is punctual if it is of the form $[t, t]$, for some $t \in \mathbb{Q}$; 
we will often represent it as $t$. An interval is bounded if both of its endpoints are rational numbers (i.e. neither $\infty$ or $-\infty$). 
Although time points are rational numbers, we will sometimes use dates instead. For example, $10/9/2023$ (formally it can be treated as the number of days that passed since $01/01/0000$).
When it is clear from the context, we may abuse distinction between intervals (i.e. sets of time points) and their representation $\langle t_1, t_2\rangle$.

Objects are represented by terms: 
a term is either a variable or a constant. 
A relational atom is an expression $R(\bm{t})$, where $R$ is a predicate and $\bm{t}$ is a tuple of terms of arity matching $R$. 
Metric atoms extend relational atoms by allowing for operators from metric temporal logic (MTL), namely, $\boxplus_\varrho$, $\boxminus_\varrho$,
$\diamondplus_\varrho$, 
$\diamondminus_\varrho$, 
$\mathcal{U}_\varrho$, and $\mathcal{S}_\varrho$, with arbitrary intervals $\varrho$.
Formally, metric atoms, $M$, are   generated by the  grammar
\begin{align*}
    M ::= &\bot \mid \top \mid R(\bm{t}) \mid \boxplus_{\varrho}M \mid \boxminus_{\varrho}M \mid \diamondplus_{\varrho}M \mid \diamondminus_{\varrho}M \mid {}\\ 
    & M\mathcal{U}_{\varrho}M \mid M \mathcal{S}_{\varrho} M,
\end{align*}
where $\top$ and $\bot$ are constants for, respectively,   truth and falsehood, whereas $\varrho$ are any intervals containing only non-negative rationals. 
A DatalogMTL rule, $r$,  is of the form
\begin{align*}
    M' \leftarrow M_1 \land M_2\land \dots \land M_n, \quad \text{for $n \geq 1$},
\end{align*}
where  each  $M_i$ is a metric atom and $M'$ is a metric atom not mentioning $\bot$, $\diamondplus$, $\diamondminus$, $\mathcal{U}$, and   $\mathcal{S}$. 
Note that similarly as ~\citet{DBLP:conf/aaai/WalegaZWG23}, we do not allow $\bot$ in $M'$ to ensure consistency and focus on derivation of facts.
We call $M'$ the head of $r$ and the set ${\{M_i \mid i \in \{1, \dots , n\} \}}$ the body of $r$, and represent them as $head(r)$ and $body(r)$, respectively. 
For an example of a DatalogMTL rule see Rule~\eqref{rule:ex}.

We call a predicate intensional (or \textit{idb}) in $\Pi$ if it  appears in the head of some rule in $\Pi$, and we call it extensional (or \textit{edb}) in $\Pi$ otherwise. 
Rule $r$ is safe if each variable in $r$'s head appears also in $r$'s body. 
A DatalogMTL program, $\Pi$, is a finite set of safe rules. 
An expression (e.g. a relational atom, a rule, or a program) is ground if it does not mention any variables. 
A fact is of the form $R(\bm{t})@\varrho$, where $R(\bm{t})$ is a ground relational atom and $\varrho$ is an  interval. 
A query is  of a similar form as a fact, but its relational atom $R(\bm{t})$ does not need to be ground.
A dataset, $\mathcal{D}$, is a finite set of facts. 
A dataset or program is bounded, if all the intervals they mention are bounded.

\subsubsection{Semantics of DatalogMTL.}

A DatalogMTL interpretation $\mathfrak{I}$ is a function which maps each time point $t \in \mathbb{Q}$ to a set of ground relational atoms (namely to atoms which hold at $t$).
If $R(\bm{t})$ belongs to this set, we write ${\mathfrak{I}, t \models R(\bm{t})}$. 
This extends to metric atoms as presented in Table~\ref{tab:semantable}.
\setlength{\tabcolsep}{3pt}
\begin{table}[b]
    \centering
\begin{tabular}{|lll|}
    \hline
    $\mathfrak{I}, t  \vDash \top$& & \text{for every $t \in \mathbb{Q}$}  \\
    $\mathfrak{I}, t  \vDash \bot$ && \text{for no $t \in \mathbb{Q}$}\\
     $\mathfrak{I}, t \vDash \boxplus_\varrho M$ & \text{iff} &
    \text{$\mathfrak{I}, t_1 \vDash M$ for all
    $t_1 $ s.t. $t_1-t \in \varrho$}\\
     $\mathfrak{I}, t \vDash \boxminus_\varrho M$ & \text{iff}&  \text{$\mathfrak{I}, t_1 \vDash M$ for all $t_1$ s.t. $t-t_1 \in \varrho$}\\
     $\mathfrak{I}, t \vDash \diamondplus_\varrho M$ & \text{iff}&  \text{$\mathfrak{I}, t_1 \vDash M$ for some $t_1 $ s.t. $t_1-t \in \varrho$}\\
     $\mathfrak{I}, t \vDash \diamondminus_\varrho M$ & \text{iff}&  \text{$\mathfrak{I}, t_1 \vDash M$ for some $t_1 $ s.t. $t-t_1 \in \varrho$}\\
     $\mathfrak{I}, t \vDash M_2 \mathcal{U}_\varrho M_1$ & \text{iff} &  \text{$\mathfrak{I}, t_1 \vDash M_1$ for some  $t_1$ s.t. $t_1 - t \in \varrho$,} \\
    &&     \text{and $\mathfrak{I}, t_2 \vDash M_2$ for all  $t_2 \in (t, t_1)$}\\    
     $\mathfrak{I}, t \vDash M_2 \mathcal{S}_\varrho M_1$ & \text{iff} &  \text{$\mathfrak{I}, t_1 \vDash M_1$ for some $t_1$ s.t. $t - t_1 \in \varrho$,} \\ 
    &&  \text{and $\mathfrak{I}, t_2 \vDash M_2$ for all $t_2 \in (t_1, t)$}
    \\\hline
\end{tabular}
    \caption{Semantics for ground metric atoms}
    \label{tab:semantable}
\end{table}

Interpretation $\mathfrak{I}$ satisfies a fact $R(\bm{t})@\varrho$, if ${\mathfrak{I}, t \vDash R(\bm{t})}$ for each ${t \in \varrho}$,
and $\mathfrak{I}$ satisfies  a set $B$ of ground metric atoms, in symbols ${\mathfrak{I}, t \vDash B}$, if ${\mathfrak{I}, t \vDash M}$ for each ${M \in B}$.
Interpretation $\mathfrak{I}$ satisfies a ground rule $r$ if ${\mathfrak{I}, t \vDash body(r)}$ implies ${\mathfrak{I}, t \vDash head(r)}$ for each ${t \in \mathbb{Q}}$. 
A ground instance of $r$ is any ground rule $r'$ such that there is a substitution $\sigma$ mapping variables into constants so that $r' = \sigma(r)$ (for any expression $e$, we will use  $\sigma(e)$ to represent the expression obtained by applying $\sigma$ to all variables in $e$).
Interpretation $\mathfrak{I}$ satisfies a rule $r$, if it satisfies all ground instances of $r$,
and it is a model of a program $\Pi$, if it satisfies all rules of $\Pi$.
If  $\mathfrak{I}$ satisfies every fact in a dataset $\mathcal{D}$, then $\mathfrak{I}$ is a model of $\mathcal{D}$. 
$\mathfrak{I}$ is a model of a pair $(\Pi, \mathcal{D})$ if $\mathfrak{I}$ is both a model of $\Pi$ and a model of $\mathcal{D}$. 
A pair $(\Pi, \mathcal{D})$ entails a fact $R(\bm{t})@\varrho$, in symbols $(\Pi, \mathcal{D}) \models R(\bm{t})@\varrho$, if all  models of $(\Pi, \mathcal{D})$ satisfy $R(\bm{t})@\varrho$. 
Given $(\Pi,\mathcal{D})$, an answer to a query $q$ is any fact $R(\bm{t})@\varrho$ such that $(\Pi,\mathcal{D})$ entails $R(\bm{t})@\varrho$ and $R(\bm{t})@\varrho = \sigma(q)$, for some substitution $\sigma$.
An interpretation $\mathfrak{I}$ contains interpretation $\mathfrak{I}'$ if $\mathfrak{I}$ satisfies every fact that $\mathfrak{I}'$ does and $\mathfrak{I}=\mathfrak{I'}$ if they contain each other.
Each dataset $\mathcal{D}$ has a unique least interpretation
$\mathfrak{I}_{\mathcal{D}}$, which is the minimal (with respect to containment) model of $\mathcal{D}$. 

Given a program $\Pi$, 
we define the immediate consequence operator $T_\Pi$ as  a function  mapping an interpretation $\mathfrak{I}$ into the least interpretation containing
$\mathfrak{I}$ and satisfying for each time point $t$ and each
ground instance $r$ (which does not mention $\bot$ in the head) of a rule in $\Pi$ the following: 
$\mathfrak{I},t \models body(r)$ implies $T_{\Pi}(\mathfrak{I}),t \models head(r)$. Now we are ready to define the canonical interpretation for a program-dataset pair $(\Pi, \mathcal{D})$. Repeated application of $T_\Pi$ to $\mathfrak{I}_\mathcal{D}$ yields a transfinite sequence of interpretations $\mathfrak{I}_0, \mathfrak{I}_1, \dots $ such that
\begin{enumerate*}[label=(\roman*)]
    \item $\mathfrak{I}_0 = \mathfrak{I}_\mathcal{D}$,
    \item $\mathfrak{I}_{\alpha+1} = T_\Pi(\mathfrak{I}_\alpha)$ for $\alpha$ a successor ordinal, and
    \item $\mathfrak{I}_\beta = \bigcup_{\alpha < \beta}\mathfrak{I}_\alpha$ for $\beta$ a limit ordinal and  successor ordinals $\alpha$ (here union of two interpretations is the least interpretations satisfying all facts satisfied in these interpretations)
\end{enumerate*}; then, $\mathfrak{I}_{\omega_1}$, where $\omega_1$ is the first uncountable ordinal, is the canonical interpretation of $(\Pi, \mathcal{D})$, denoted as ${\mathfrak{C}_{\Pi, \mathcal{D}}}$ \cite{DBLP:journals/jair/BrandtKRXZ18}. 

\subsubsection{Reasoning Task.}
We focus on fact entailment, the main reasoning task in DatalogMTL, which involves checking whether a given pair of DatalogMTL program and dataset $(\Pi, \mathcal{D})$ entails a fact $R(\bm{t})@\varrho$. 
This task is \textsc{ExpSpace}-complete \cite{DBLP:journals/jair/BrandtKRXZ18} in combined complexity and \textsc{PSpace}-complete in data complexity \cite{DBLP:conf/ijcai/WalegaGKK19}.

\subsubsection{Magic Set Rewriting.}
We now  recapitulate magic set rewriting for Datalog.
For a detailed description of magic sets we refer the reader to the work of \citet{DBLP:books/aw/AbiteboulHV95, DBLP:books/cs/Ullman89,DBLP:journals/jcss/FaberGL07,magicPerf}.
Datalog can be seen as a fragment of DatalogMTL that disallows metric temporal operators and discards the notion of time. 
In the context of Datalog, given a query (which is now a relational atom) $q$ and a program-dataset pair $(\Pi, \mathcal{D})$,  the magic set approach  constructs a pair $(\Pi',\mathcal{D}')$ such that $(\Pi', \mathcal{D}')$  and $(\Pi, \mathcal{D})$
provide the same answers to $q$.

The first step of magic set rewriting involves program adornment (adding  superscripts to predicates). 
Adornments will guide construction of $\Pi'$.
The adorned program $\Pi_a$ 
is constructed as follows:
\begin{itemize}
    \item  Step 1: assume that the query is of the form $q= Q(\bm{t})$. We adorn  $Q$ with a string $\gamma$ of letters \textit{b} and \textit{f}, which stand for `bound' and `free' terms. The length of $\gamma$ is the arity of $Q$; 
    the $i$th element of $\gamma$ is \textit{b} if the $i$th term in $\bm{t}$ is a constant, and \textit{f} otherwise.
    For example $Q(x,y,\mathit{Arthur})$ yields $Q^{\mathit{ffb}}$.
    We set  $A= \{  Q^\gamma \}$ and $\Pi_a = \emptyset$.
    \item Step 2: we remove one adorned predicate $R^\gamma$ from $A$ (in the first application of Step 2, set $A$ contains only one element, but on  later stages $A$ will be  modified).  For each rule $r$ whose head predicate is $R$, 
    we generate an adorned rule $r_a$ and add it to $\Pi_a$.
    The rule $r_a$ is constructed from $r$ as follows.
    We start by replacing the head predicate $R$ in $r$ with $R^\gamma$. 
    Next we  label all occurrences of terms in $r$ as bound or free as follows:
    \begin{enumerate}[label=(\roman*)]
        \item terms in  $head(r)$  are labelled according to $\gamma$ (i.e. the $i$th term is bound iff the $i$th element of $\gamma$ is $b$),
        \item each constant in $body(r)$ is labelled as bound,
        \item each variable  that is labelled as bound in $head(r)$  is labelled as bound in $body(r)$,
        \item if a variable occurs in a body atom, all its occurrences in the subsequent atoms (i.e. located to the right of this body atom) are labelled as bound, 
        \item all remaining variables in $body(r)$ are labelled as free.
    \end{enumerate}
    Then each body atom $P(\bm{t}'$) in $r$ is replaced with
    $P^{\gamma'}(\bm{t}')$ such that the $i$th letter of $\gamma'$ is $b$ if the $i$th term of $\bm{t}'$ is labelled as bound, and it is $f$ otherwise.
    Moreover, if $P$ is \textit{idb} in $\Pi$ and $P^{\gamma'}$ was never constructed so far, we add $P^{\gamma'}$ to $A$.
    For example if $R^\gamma$ is $Q^{\mathit{ffb}}$ and $r$ is $Q(x,Arthur,y) \xleftarrow{} P(x, Beatrice) \wedge T(x,y)$,
    then $r_a$ is 
    $Q^{\mathit{ffb}}(x,Arthur,y) \xleftarrow{} P^{\mathit{fb}}(x, Beatrice) \wedge T^{\mathit{bb}}(x,y)$.
    Then, we keep repeating Step 2 until $A$ becomes empty. 
\end{itemize}
Next, we generate the final $\Pi'$ from the above constructed $\Pi_a$. 
For a tuple of terms $\bm{t}$ and an adornment $\gamma$ of the same length, we let $\bt{\bm{t}}{ \gamma}$ the sequences of bound terms in $\bm{t}$  according to $\gamma$.
For example, $\bt{(x,y,\mathit{Arthur}}{\mathit{bfb}}=(x,\mathit{Arthur})$.
Then, for each rule $r_a\in \Pi_a$, represented as
\begin{align*}
    R^\gamma(\bm{t}) \xleftarrow{} R_1^{\gamma_1}(\bm{t}_1) \land  
    \dots \land  R_n^{\gamma_n}(\bm{t}_n),
\end{align*}
we add to $\Pi'$ the rule
\begin{align}
     R(\bm{t}) &  \gets m\_R^\gamma(\bm{t}') \land   R_1(\bm{t}_1) \land   \cdots \land   R_n(\bm{t}_n) \label{rule1}
\end{align}
and the
following rules, for all $i \in \{1, \dots, n\}$ such that $R_i$ is an \textit{idb} predicate in $\Pi$:
\begin{align}
    m\_R_i^{\gamma_i}(\bm{t}_i')  &\gets   m\_R^\gamma(\bm{t}') \land  R_1(\bm{t}_1) \land \cdots \land R_{i-1}(\bm{t}_{i-1}), \label{rule2}
\end{align}
where $\bm{t}'=\bt{\bm{t}}{\gamma}$, $\bm{t}_i'= \bt{\bm{t}_i}{\gamma_i}$, and $m\_R^\gamma$ and $m\_R_i^{\gamma_i}$ are newly introduced `magic predicates' with arities $|\bm{t}'|$ and $|\bm{t}_i'|$, respectively. 
Intuitively, magic predicates are responsible for storing tuples relevant for the derivations of query answers. 
In particular, Rule \eqref{rule1} considers only derivations that are (i) considered by the original rule $r$ and (ii) allowed  by the magic predicate $m\_R^\gamma$. 
Rule \eqref{rule2} is responsible for deriving facts with magic predicates. 

Finally, given the input query $q= Q(\bm{t})$, we construct $\mathcal{D}'$ by adding $m\_Q^\gamma(\bt{\bm{t}}{\gamma})$ to $\mathcal{D}$,
where $\gamma$ is the adornment from Step 1.
The new pair $(\Pi',\mathcal{D}')$ provides the same answers to query $q$ as $(\Pi,\mathcal{D})$, but generally allows us to provide answers to $q$ faster than $(\Pi,\mathcal{D})$.

As an example of application of magic set rewriting, consider a Datalog dataset $\mathcal{D}$, a Datalog version of Rule~\eqref{rule:ex}:
\begin{align}\label{datalogrule:ex}
     P(x) \xleftarrow{} I(x, y) \wedge P(y),
\end{align}
and a query $P(\mathit{Arthur})$. 
After applying Step 1 we get $A=\{P^b\}$ and $\Pi_a = \emptyset$. 
Then, in Step 2, we remove $P^b$ from $A$, and adorn Rule~\eqref{datalogrule:ex} according to conditions (i)--(v).
As a result, we obtain the following rule:
\begin{align*}
    P^b(x) \xleftarrow{} I^{bf}(x, y) \wedge P^b(y)
\end{align*}
as a rule in $\Pi_a$.
No predicate is added to $A$, so $A = \emptyset$, and so we do not apply Step 2 any more.
Next, we generate the following program $\Pi'$ from $\Pi_a$:
\begin{align*}
    P(x) &\xleftarrow{} m\_P^b(x) \wedge I(x, y) \wedge P(y), \\
    m\_P^b(y)&\xleftarrow{} m\_P^b(x) \wedge I(x, y).
\end{align*}
Finally, we set $\mathcal{D}' = \mathcal{D} \cup \{m\_P^b(\mathit{Arthur})\}$, which
concludes the construction of $(\Pi',\mathcal{D}')$.
The newly constructed  $(\Pi',\mathcal{D}')$ and the initial $(\Pi,\mathcal{D})$ have  the same answers to the query $P(\mathit{Arthur})$.
However, $(\Pi',\mathcal{D}')$ entails a smaller amount of facts about $P$.
Indeed, if  in the dataset $\mathcal{D}$ there is no chain of interactions $I$ connecting a constant $Beatrice$ with $Arthur$, then neither $P(Beatrice)$ nor $m\_P^b(Beatrice)$ will be  entailed by $(\Pi',\mathcal{D}')$.

\section{Magic Sets for DatalogMTL}

In this section, we explain how we extend magic set rewriting to the DatalogMTL setting. 
Since the approach is relatively complex, we start by providing  an example how, given a particular query $q$ and a  DatalogMTL program-dataset pair $(\Pi, \mathcal{D})$, we construct  $(\Pi', \mathcal{D}')$. 
For simplicity, in the example we will use  a ground query $q$ (i.e. a fact), but it is worth emphasising that our approach also supports non-ground queries.

\subsubsection{Example.}
Consider a query $q=P(\mathit{Arthur}) @  10/9$,
a DatalogMTL dataset $\mathcal{D}$, and a program consisting of the rules
\begin{align*}
    \boxplus_{[0, 2]} P(x) & \gets  I(x, y) \wedge P(y), &(r_1)\\
    \boxplus_{[0, 1]} P(x) & \gets  I(x, y) \wedge \diamondminus_{[0, 1]} P(y). &(r_2)
\end{align*}
We construct $\mathcal{D}'$ 
by adding ${m\_P^{b}(\mathit{Arthur}) @  10/9}$ to $\mathcal{D}$.
Hence, the construction of $\mathcal{D}'$ is similar to the case of (non-temporal) Datalog, but now the  fact with magic predicate mentions also the time point from the input query.
This will allow us not only to determine which atoms, but also which time points are relevant 
for answering the query. 

Next, we construct an adorned  program $\Pi_a$ as below:
\begin{align*}
    \boxplus_{[0, 2]} P^{b}(x) & \gets  I^{bf}(x, y) \wedge P^{b}(y), &(r_{a_1})\\
    \boxplus_{[0, 1]} P^{b}(x) & \gets  I^{bf}(x, y) \wedge \diamondminus_{[0, 1]} P^{b}(y). &(r_{a_2})
\end{align*}
The construction of $\Pi_a$ is analogous to that in Datalog (see the previous section), but $\Pi_a$ now contains temporal operators.
As in the case of Datalog, we will obtain $\Pi'$ by 
constructing two types of rules, per a rule in $\Pi_a$. 
However, the construction of these rules
will require further modifications.

Rules of the first type are extensions of rules in $\Pi_a$ obtained by adding atoms with  magic predicates to rule bodies.
We can observe, however, that the straightforward adaptation of Rule~\eqref{rule1} used in Datalog is not appropriate in the temporal setting.
To observe an issue, let us consider Rule~$(r_{a_1})$.
The na\"ive
adaptation of Rule~\eqref{rule1}
would construct from Rule~$(r_{a_1})$ the following rule:
$$
\boxplus_{[0, 2]} P(x)  \gets m\_P^{b}(x) \wedge I(x, y) \wedge P(y).
$$
Recall that our goal is to construct $(\Pi',\mathcal{D}')$ which will have exactly the same answers to the input query as the original $(\Pi,\mathcal{D})$.
The rule above, however, would not allow us to achieve it.
Indeed, if 
$\mathcal{D}$ consists of facts $I(Arthur, Beatrice) @ 8/9$ and  $P(Beatrice)@8/9$,
then the original Rule ($r_1$) allows us to derive $P(\mathit{Arthur}) @  [8/9,10/9]$, and so, the query fact $P(\mathit{Arthur}) @  10/9$ is derived.
However, the  rule constructed above does not allow us to derive any new fact from $\mathcal{D}'$, because we do not have ${m\_P^{b}(\mathit{Arthur}) @  8/9}$ in $\mathcal{D}'$, which is required to satisfy the rule body.
Our solution is to 
use ${\diamondplus_{[0, 2]} m\_P^{b}(x)}$ 
instead of 
$m\_P^{b}(x)$ in the rule body, namely we construct the rule 
\begin{align*}
    \boxplus_{[0, 2]} P(x) \gets \diamondplus_{[0, 2]}\mathit{m\_P}^{b}(x) \wedge I(x, y) \wedge P(y)   
    .
\end{align*}
This allows us to obtain missing derivations;
in particular, in our example above, this new rule allows us to derive $P(\mathit{Arthur}) @  [8/9,10/9]$, and so, $P(\mathit{Arthur}) @  10/9$, as required.
Similarly, for Rule~$(r_{a_2})$ we construct 
\begin{align*}
    \boxplus_{[0, 1]} P(x) \gets \diamondplus_{[0, 1]} m\_P^b(x) \wedge I(x, y) \wedge \diamondminus_{[0, 1]}P(y).
\end{align*}

Next we  construct rules of the second type,  by adapting the idea from Rule~\eqref{rule2}.
For Rule~$(r_{a_1})$ we use the form of
Rule~\eqref{rule2}, but we additionally precede the magic predicate in the body with the diamond operator, to prevent  the observed issue with missing derivations.
Hence, we obtain
\begin{align*}
\mathit{m\_P}^{b}(y) \gets \diamondplus_{[0, 2]} m\_P^{b}(x) \wedge  I(x, y).
\end{align*}
The case of Rule~$(r_{a_2})$, however, is more challenging as this rule mentions a temporal operator in the body.
Because of that, we observe that the below rule would not be appropriate
\begin{align*}
 m\_P^b(y) \gets \diamondplus_{[0,1]} m\_P^b(x)  \wedge I(x, y) .
\end{align*}
To observe the issue, let us consider a program $\Pi$, which instead of Rule $(r_1)$ has Rule $(r_3)$ which is of the form  $P(x) \gets S(x)$.
Therefore, Rule $(r_{a_3})$ is  $P^b(x) \gets S^b(x)$
and its first-type-rule   is $P(x) \gets m\_P^b(x) \wedge S(x)$.
Assume moreover, that $\mathcal{D}$ consists of $S(Beatrice)@8/9$ and $I(Arthur,Beatrice)@9/9$.
Hence,  Rule $(r_3)$ allows us to derive $P(Beatrice)@8/9$, and then Rule $(r_2)$ derives $P(Arthur)@[9/9,10/9]$. 
Although with the rule generated above from Rule $(r_{a_2})$, we can 
derive from $\mathcal{D}'$ the fact
$m\_P^{b}(Beatrice)@9/9$, we cannot derive  $m\_P^{b}(Beatrice)@8/9$.
This, in turn, disallows us to  derive $P(Arthur)@10/9$.
In order to overcome this problem, it turns out that it suffices to add a temporal operator in the head of the rule.
In our case, we  construct the following rule:
\begin{align*}
    \boxminus_{[0, 1]}m\_P^b(y) \gets \diamondplus_{[0,1]} m\_P^b(x)  \wedge I(x, y).
\end{align*}
Hence, the finally constructed $\Pi'$ 
consists of the  rules
\begin{align*}
    \boxplus_{[0, 2]} P(x)& \gets \diamondplus_{[0, 2]}\mathit{m\_P}^{b}(x) \wedge I(x, y) \wedge P(y),
\\
    \boxplus_{[0, 1]} P(x) &\gets \diamondplus_{[0, 1]} m\_P^b(x) \wedge I(x, y) \wedge \diamondminus_{[0, 1]}P(y),
\\
\mathit{m\_P}^{b}(y) &\gets \diamondplus_{[0, 2]} m\_P^{b}(x) \wedge  I(x, y),
\\
\boxminus_{[0, 1]}m\_P^b(y) & \gets \diamondplus_{[0,1]} m\_P^b(x)  \wedge I(x, y).
\end{align*}
In what follows we will provide details of this construction and 
show that 
$(\Pi',\mathcal{D}')$
and $(\Pi,\mathcal{D})$
have the same answers to the input query.

\subsubsection{Algorithm Overview.}
Our approach to constructing $(\Pi',\mathcal{D}')$ is provided in
Algorithm \ref{algo:magic}.
The algorithm takes as input a DatalogMTL program $\Pi$, dataset $\mathcal{D}$, and  query $q=Q(\bm{t})@\varrho$, and it returns the rewritten pair $(\Pi', \mathcal{D}')$. 
For simplicity of presentation we will  assume that rules in $\Pi$ do not have nested temporal operators and that their heads always mention either $\boxplus$ or $\boxminus$.
Note that these assumptions are without loss of generality, as we can always flatten nested operators by introducing additional rules, whereas an atom $M$ with no temporal operators can be written as $\boxplus_{0} M$ (or equivalently as $\boxminus_{0} M$)~\cite{DBLP:journals/jair/BrandtKRXZ18}.
Note, moreover, that our rewriting technique does not require the input query to be a fact, namely the sequence $\bm{t}$ in the query
$Q(\bm{t})@\varrho$ can contain variables. 

Line 1 of Algorithm~\ref{algo:magic} computes $\mathcal{D}'$ as an extension of $\mathcal{D}$ with the fact $m\_Q^{\gamma_0}(\bt{\bm{t}}{\gamma_0})@\varrho$. 
It is obtained from the query $Q(\bm{t})@\varrho$ using a sequence
$\gamma_0$ of $b$ and $f$ such that the $i$th element is $b$ if and only if the $i$th element of $\bm{t}$ is a constant.

Lines 2--3  compute the adorned program $\Pi_a$. 
In particular, Line 2 initialises the set $A$  of adorned \textit{idb} predicates that are still to be processed and the set ${\Pi_a}$ of  adorned rules.
Then, Line 3 constructs $\Pi_a$ in the same way as in the case of Datalog (see Preliminaries).

Line 4 initialises, as an empty set, $\Pi'$ and two auxiliary sets of rules $\Pi'_1$ and  $\Pi'_2$.
Then, Lines 5--9 
compute a set $\Pi_1'$ of rules of the first type and 
Lines 10--15 compute a set $\Pi_2'$ of rules of the second type (to be precise, $\Pi_1'$ and $\Pi_2'$ become rules of the first and the second type after removing adornments in Line 17).

In particular, the loop from Lines 5--9 constructs from each $r \in \Pi_a$ a rule $r'$ and adds it to $\Pi_1'$.
To construct $r'$,  Lines 7--8 check if the head of $r$  mentions $\boxplus$ or $\boxminus$. 
In the first case,  Line 7 adds an atom with $\diamondplus$ and a magic predicate to the rule body.
In the second case,  Line 8 adds an atom with $\diamondminus$ and a magic predicate to the rule body.
Then, Line 9 adds $r'$ to $\Pi_1'$.

Lines 10--15 construct rules $\Pi_2'$ of the second type from rules in $\Pi_1'$.
To this end, for each rule $r$ in $\Pi_1'$ (Line 10)
and each body atom $M$ in $r$ which mentions an \textit{idb} predicate (Line 11) the algorithm performs the following computations.
In Line 12 it computes a set $H$ of atoms with magic predicates which will be used to construct rule heads.
This is performed with a function \textbf{MagicHeadAtoms} implemented with 
Algorithm~\ref{algo:MagicHeadAtoms}, which we will describe later.
Then, the loop in Lines 13--15
computes a rule $r'$ for each atom $M'$ in $H$ and adds this $r'$ to $\Pi_2'$.
Rule $r'$ is constructed from $r$ in Line 14, by replacing the head with $M'$ and deleting from the rule body $M$ as well as all the other body atoms located to the right of $M$.

To construct the final program $\Pi'$ from $\Pi_1'$ and $\Pi_2'$, Line~16 constructs the union of these two sets and Line~17 deletes from these rules all the adornments of predicates which are non-magic, that is, not preceded by $m\_$.
Line~18 returns the pair consisting of $\Pi'$ and $\mathcal{D}'$.

\begin{algorithm}[t]
\SetAlgoNoLine
\SetAlgoNoEnd
\SetKwInOut{Input}{Input}
\SetKwInOut{Output}{Output}

\Input{A DatalogMTL program $\Pi$, dataset $\mathcal{D}$, and query $q = Q(\bm{t})@\varrho$}

\Output{A DatalogMTL program $\Pi'$ and  dataset $\mathcal{D}'$}

$\mathcal{D}' := \mathcal{D}\cup
\{ m\_Q^{\gamma_0}(\bt{\bm{t}}{\gamma_0})@\varrho \}$ ;

$A:= \{ Q^{\gamma_0} \}$ ; \quad  $\Pi_a := \emptyset$ \; 

$\Pi_a :=$ a program obtained from $\Pi$ by applying  Step 2 from Section ``Magic Set Rewriting'' until $A = \emptyset$ ;

$\Pi' := \emptyset$ ; \quad $\Pi'_1 := \emptyset$ ; \quad $\Pi'_2 := \emptyset$ ;

\ForEach{rule $r \in \Pi_a$}{
Let  $\Box_{\varrho'}R^\gamma(\bm{t})$, with $\Box \in \{\boxplus,\boxminus \}$, be the head of $r$ ;

    \lIf{$\Box=\boxplus$}{
$r':= r$ with $\diamondplus_{\varrho'}m\_R^\gamma(\bt{\bm{t}}{\gamma})$ added as the first body atom
    }
    \lIf{$\Box=\boxminus$}{
$r':= r$ with $\diamondminus_{\varrho'}m\_R^\gamma(\bt{\bm{t}}{\gamma)}$ added as the first body atom
    }

    $\Pi'_1 := \Pi'_1 \cup \{r'\}$ ;
}

\ForEach{rule $r \in \Pi'_1$}{
    \ForEach{$M \in body(r)$ with an \textit{idb} predicate}{
    $H := \textbf{MagicHeadAtoms}(M)$ ;

        \ForEach{atom $M' \in H$}{
            $r' := r$ with head replaced by $M'$ and body modified by deleting $M$ and all body atoms to the right of $M$ ;
            
            $\Pi_2' := \Pi_2' \cup \{r'\}$ ;
        }
    }
}
$\Pi' =\Pi'_1 \cup \Pi_2'$ ;

Delete adornments from  non-magic predicates in $\Pi'$ ;

\KwRet $(\Pi', \mathcal{D}')$\;
\caption{\textbf{MagicRewriting}}
\label{algo:magic}
\end{algorithm}

\begin{algorithm}[t]
\SetKwInOut{Input}{Input}
\SetKwInOut{Output}{Output}
\SetAlgoNoLine
\SetAlgoNoEnd

\Input{A metric atom $M$}
\Output{A set $H$ of metric atoms with magic predicates}

$H := \emptyset$;  \quad $M' := M$ ;

\ForEach{adorned \textit{idb} relational atom $R^\gamma(\bm{t})$ in $M$}{

Replace   $R^\gamma(\bm{t})$ in $M'$ with $m\_R^\gamma( \bt{\bm{t}}{\gamma})$ ;
}
\lIf{$M' = \boxplus_{\varrho} M_1$ or $M' = \boxminus_{\varrho} M_1$}{ $H := \{M'\}$ }

\lIf{$M' = \diamondplus_{\varrho} M_1$}{$H := \{\boxplus_{\varrho} M_1\}$ }
\lIf{$M' = \diamondminus_{\varrho} M_1$}{$H := \{\boxminus_{\varrho} M_1\}$ }

\If{$M' = M_2\ \mathcal{S}_{\varrho} M_1 $}{ 
$t_{max} := \textnormal{the right endpoint of }\varrho $ ;

\If{$M_2$ contains a magic predicate}{$H := H \cup \{\boxminus_{[0,t_{max}]} M_2\} $ ; }
\If{$M_1$ contains a magic predicate}{$H := H \cup \{\boxminus_{\varrho} M_1$\} ;}} 

\If{$M' = M_2\ \mathcal{U}_{\varrho} M_1$}{ 
$t_{max} := \textnormal{the right endpoint of }\varrho $ ;

\If{$M_2$ contains a magic predicate}{$H := H \cup \{\boxplus_{[0, t_{max}]} M_2\}$ ;}
\If{$M_1$ contains a magic predicate}{$H := H \cup \{\boxplus_{\varrho} M_1$\} ;}} 

\KwRet $H$ ;
\caption{\textbf{MagicHeadAtoms}}
\label{algo:MagicHeadAtoms}
\end{algorithm}

\subsubsection{MagicHeadAtoms.}
The function \textbf{MagicHeadAtoms} is used to construct rules of the second type (i.e. the set $\Pi_2'$), as presented in
Algorithm~\ref{algo:MagicHeadAtoms}. 
The algorithm takes as input a metric atom $M$ and returns a set $H$ of metric atoms, whose predicates are magic predicates.
Line 1 initialises $H$ as the empty set and introduces an auxiliary $M'$ which is set to $M$. 
The loop in Lines 2--3 modifies $M'$ by replacing \textit{idb} predicates with magic predicates. 
Then, Lines 4-18 construct the set $H$ depending on the form of $M'$.
Line 4 consider the cases when $M'$ mentions $\boxplus$ or $\boxminus$.
Line 5 considers the case when $M'$ mentions $\diamondplus$, and Line 6 when $M'$ mentions $\diamondminus$.
Lines~7-12 are for the case when $M'$ mentions $\mathcal{S}$ and Lines~13-18 
when $M'$ mentions $\mathcal{U}$.
The final set $H$ is returned in Line 19.

\subsubsection{Correctness of the Algorithm.}
Next we aim to show that, given a program $\Pi$, a dataset ${\mathcal{D}}$, and a query $q=Q(\bm{t})@\varrho$, the  pair $(\Pi', \mathcal{D}')$ returned by Algorithm~\ref{algo:magic} entails the same answers to $q$ as $(\Pi, \mathcal{D})$. This result is provided by the following theorem.
\begin{theorem}
\label{theorem1}
Let $\Pi$ be a
 DatalogMTL program,  $\mathcal{D}$ a dataset, $Q(\bm{t})@\varrho$ a query,
 and $(\Pi', \mathcal{D}')$  the output of Algorithm~\ref{algo:magic} when run on  $\Pi$, $\mathcal{D}$, and $Q(\bm{t})@\varrho$.
 For 
 each time point $t \in \varrho$ and each substitution $\sigma$ mapping variables in $\bm{t}$ to constants we have
\begin{align*}
    (\Pi, \mathcal{D})\models  Q(\sigma(\bm{t}))@t  && \text{iff} &&
    (\Pi',\mathcal{D}') \models Q(\sigma(\bm{t})) @t.
\end{align*}
\end{theorem}
\begin{psketch}
    For the if part, we leverage the fact that the first kind of rules in $\Pi'$ (i.e. rules in $\Pi_1'$)  are constructed by adding body atoms with magic predicates to the bodies of rules in $\Pi$.
    Hence, we can show that  each  fact entailed by $(\Pi', \mathcal{D}')$ is also entailed by $(\Pi, \mathcal{D})$.

   The opposite direction is more challenging. 
   The main part of the proof is to show that if $(\Pi,\mathcal{D}) \models R(\bm{t})@t$ and $(\Pi',\mathcal{D}') \models m\_R^\gamma(\bt{\bm{t}}{\gamma})@t$ for some $R(\bm{t})@t$ and $\gamma$,
    then $(\Pi',\mathcal{D}') \models R(\bm{t})@t$.
    This is sufficient to finish the proof.
    Indeed,
    note that by the construction of $\mathcal{D}'$, it contains $m\_Q^{\gamma_0}(\bt{\bm{t}}{\gamma_0})@\varrho$. 
    Therefore, by the implication above we have $(\Pi',\mathcal{D}') \models m\_Q^{\gamma_0}(\bt{\bm{t}}{\gamma_0})@t$, as required.
    To show the needed implication we conduct a  proof by a transfinite induction on the number of $T_\Pi$ applications to $\mathfrak{I}_\mathcal{D}$ and $T_{\Pi'}$ applications to $\mathfrak{I}_{\mathcal{D}'}$.
\end{psketch}

\subsubsection{Reasoning with the Algorithm.}
Theorem~\ref{theorem1} shows us  that
to answer a query $q$ with respect to $\Pi$ and $\mathcal{D}$, we can instead answer $q$ with respect to $\Pi'$ and $\mathcal{D}'$ constructed by  Algorithm~\ref{algo:magic}.
In what follows we will show how to efficiently answer
$q$ with respect to $\Pi'$ and $\mathcal{D}'$.
For this, we will use a recently introduced approach by~\citet{DBLP:conf/aaai/WalegaZWG23}, which  constructs a finite representation of (usually infinite) canonical model of a program and a dataset. 
Importantly, this approach is guaranteed to terminate, which stands in contrast with a pure materialisation approach in DatalogMTL that  often requires a transfinite number of rule applications.

Although the approach of~\citet{DBLP:conf/aaai/WalegaZWG23}, presented in their Algorithm 1,
overperforms other reasoning techniques for DatalogMTL, it is worth observing that it is guaranteed to terminate only if the input program and dataset are both bounded (i.e. mention only bounded intervals).
Moreover, the runtime of their algorithm heavily depends on the distance between the minimal and maximal time points in the input dataset.
This makes applying their algorithm  in our setting challenging, because time points of our
$\mathcal{D}'$ depend on the input query.
Indeed,
$\mathcal{D}'$ constructed in Line 1 of our  Algorithm~\ref{algo:magic} 
contains the fact $m\_Q^{\gamma_0}(\bt{\bm{t}}{\gamma_0})@\varrho$, whose interval $\varrho$ is copied from the input query $q = Q(\bm{t})@\varrho$.
Hence, if $\varrho$ is large, then applying the approach of~\citet{DBLP:conf/aaai/WalegaZWG23} may be inefficient.
Moreover, if $\varrho$ is unbounded, then so is $\mathcal{D}'$, and  we loose the guarantee of termination.

We will address these challenges below.
First, we observe that given a program $\Pi$, dataset $\mathcal{D}$, and query $q=Q(\bm{t})@\varrho$, 
instead of computing 
$\textbf{MagicRewriting}(\Pi,\mathcal{D},Q(\bm{t})@\varrho)$
we can compute 
$\textbf{MagicRewriting}(\Pi,\mathcal{D},Q(\bm{t})@(-\infty,+\infty))$
and the new program-dataset pair will still entail the same answers to $q$ as the the input $\Pi$ and $\mathcal{D}$.
This is a simple consequence of Theorem~\ref{theorem1} as stated formally below.

\begin{corollary}
\label{coro1}
Let $\Pi$ be a
 DatalogMTL program,  $\mathcal{D}$ a dataset, $Q(\bm{t})@\varrho$ a query,
 and $(\Pi', \mathcal{D}')$  the output of Algorithm~\ref{algo:magic} when run on  $\Pi$, $\mathcal{D}$, and $Q(\bm{t})@(-\infty,+\infty)$.
 For 
 each time point $t \in \varrho$ and each substitution $\sigma$ mapping variables in $\bm{t}$ to constants we have
\begin{align*}
    (\Pi, \mathcal{D})\models  Q(\sigma(\bm{t}))@t  && \text{iff} &&
    (\Pi',\mathcal{D}') \models Q(\sigma(\bm{t})) @t.
\end{align*}
\end{corollary}

We note that $\textbf{MagicRewriting}(\Pi,\mathcal{D},Q(\bm{t})@(-\infty,+\infty))$ constructs a pair $(\Pi',\mathcal{D}')$
such that $\mathcal{D}'$ contains $m\_R(\bm{t})@(-\infty,+\infty)$.
Hence,  $\mathcal{D}'$ is not bounded,
and so, application of Algorithm 1 by~\citet{DBLP:conf/aaai/WalegaZWG23} to $\Pi'$, $\mathcal{D}'$, and $q$ is not guaranteed to terminate.
What we show next, however, is that unboundedness of $m\_R(\bm{t})@(-\infty,+\infty)$ does not lead to non-termination.
Indeed,  if $\Pi$ and $\mathcal{D}$ are bounded, then the algorithm of~\citet{DBLP:conf/aaai/WalegaZWG23} is guaranteed to terminate  on $\Pi'$, $\mathcal{D}'$, and $q$, as stated next.

\begin{theorem}
\label{migrationProp}
Let $\Pi$ be a bounded program, $\mathcal{D}$ a bounded dataset, $q=Q(\bm{t})@\varrho$ a fact, and
$(\Pi', \mathcal{D}')$ the output of Algorithm~\ref{algo:magic} when run on  $\Pi$, $\mathcal{D}$, and $Q(\bm{t})@(-\infty, +\infty)$.
Application of Algorithm 1 by~\citet{DBLP:conf/aaai/WalegaZWG23} is guaranteed to terminate on $\Pi'$, $\mathcal{D}'$, and $q$, even though $\mathcal{D}'$ is not bounded.
\end{theorem}
\begin{psketch}
Observe that the program $\Pi'$ constructed by our Algorithm~\ref{algo:magic} is bounded and the only  unbounded fact in $\mathcal{D}'$ is  $m\_Q^{\gamma_0}(\bt{\bm{t}}{\gamma_0})@(-\infty, +\infty)$.
To show that application of Algorithm 1 by~\citet{DBLP:conf/aaai/WalegaZWG23} terminates on $\Pi'$, $\mathcal{D}'$, and $q$, we  replace the unbounded fact in $\mathcal{D}'$ with an additional rule in $\Pi'$, namely we construct $\Pi''$ by extending $\Pi'$ with a single rule $m\_Q^{\gamma_0}(\bt{\bm{t}}{\gamma_0}) \gets \top$, which simulates $m\_Q^{\gamma_0}(\bt{\bm{t}}{\gamma_0})@(-\infty, +\infty)$.
We can show that the number of iterations of their algorithm running on $(\Pi'', \mathcal{D}, q)$ is equal to or greater than the number of its iterations on
$(\Pi', \mathcal{D}', q)$ by one.
Since $\Pi''$ and $\mathcal{D}$ are  bounded, both of the  above mentioned numbers of iterations need to be finite.
Hence, the algorithm is guaranteed to terminate on $\Pi'$, $\mathcal{D}'$, and $q$, as stated in the theorem.
\end{psketch}

Hence, we obtain a terminating approach for reasoning in DatalogMTL with magic set rewriting.
To sum up, our approach to checking entailment of a fact $q = Q(\bm{t})@\varrho$ with respect to a program $\Pi$ and a dataset $\mathcal{D}$ consists of two steps.
The first step is to  construct $\Pi'$ and $\mathcal{D}'$ by applying our Algorithm~\ref{algo:magic} to $\Pi$, $\mathcal{D}$, and $Q(\bm{t})@(-\infty,+\infty)$.
The second step is to apply Algorithm 1 of~\citet{DBLP:conf/aaai/WalegaZWG23} to $\Pi'$, $\mathcal{D}'$, and $q$.
This method is guaranteed to terminate and correctly answer whether $\Pi$ and $\mathcal{D}$ entail $q$.
In the next section we will study performance of this new approach.

\section{Empirical Evaluation}
We implemented our algorithm and evaluated it on LUBM$_t$ ~\cite{DBLP:conf/aaai/WangHWG22}---a temporal version of LUBM~\cite{DBLP:journals/ws/GuoPH05LUBM}---, on  iTemporal~\cite{DBLP:conf/icde/BellomariniNS22}, and on the meteorological~\cite{Weather} benchmarks.
The aim of our experiments is to check how the performance of reasoning with magic programs compare with  reasoning using original programs. 
The three  programs we used have 85, 11, and 4 rules, respectively. 
Moreover, all our test queries are facts (i.e. they do not have free variables) since current DatalogMTL system do not allow for answering queries with variables (they are designed to check entailment). 
Given a program $\Pi$ and a dataset $\mathcal{D}$ we used  MeTeoR system~\cite{DBLP:conf/aaai/WalegaZWG23} to perform reasoning. 
The engine  first loads the pair and preprocesses the dataset before it starts  reasoning. Then, for each given query fact, we record the wall-clock time that the baseline approach~\cite{DBLP:conf/aaai/WalegaZWG23} takes to answer the query.
For our approach, we record the time taken for running Algorithm~\ref{algo:magic} to generate the transformed pair $(\Pi', \mathcal{D}')$ plus the time of answering the query via the materialisation of $(\Pi', \mathcal{D}')$. 
We ran the experiments on a server with 256 GB RAM, Intel Xeon Silver 4210R CPU @ 2.40GHz, Fedora Linux 40, kernel version 6.8.5-301.fc40.x86\_64. In each test case, we verified that both our approach and the baseline approach produced the same answer to the same query. Moreover, we observed that the time taken by running Algorithm~\ref{algo:magic} was in all test cases less than 0.01 second, so we do not report these numbers separately. 
Note that our results are not directly comparable to those presented by~\citet{DBLP:conf/aaai/WalegaZWG23} and~\citet{DBLP:conf/aaai/WangHWG22}, since the datasets and queries are generated afresh, and different hardware is used. The datasets we used, the source code of our implementation, as well as an extended technical report, are available online.\footnote{https://github.com/RoyalRaisins/Magic-Sets-for-DatalogMTL}
\subsubsection{Comparison With Baseline.}
We compared the performance of our approach to that of the reasoning with original programs on datasets from LUBM$_t$ with $10^6$ facts, iTemporal with $10^5$ facts, and ten years' worth of meteorological data. 
We ran 119, 35, and 28 queries on them respectively. 
These queries cover every \textit{idb} predicate and are all facts.
The results are summarised in the box plot in Figure \ref{complubm}. 
The red and orange boxes represent time statistics for magic programs and the original program, respectively, while the blue box represents the per-query time ratios. 
Since computing finite representations using the baseline programs requires too much time for LUBM$_t$ and iTemporal (more than 100,000s and 50,000s) and memory and time consumption should theoretically stay the same across each computation, we used an average time for computing finite representations as the query time for each not entailed query instead of actually running them. 

For the LUBM$_t$ case, we achieved a performance boost on every query, the least enhanced (receiving the least speedup among other queries compared with the original program) of these received a 1.95 times speedup. 
For the entailed queries, half were accelerated by more than 3.46 times, 25\% were answered more than 112 times faster. 
The most enhanced query was answered 2,110 times faster. 
For the not entailed queries, every query was accelerated by more than 12 times. 
Three quarters of the queries were answered more than 128 times faster using magic programs than the original program, 25\% were answered more than 9,912 times faster. 
The most enhanced query was answered 111,157 times faster. 
It is worth noting that for the queries where a performance uplift are most needed, which is when the original program's finite representation needs to be computed, 
especially when the query fact is not entailed, our rewritten pairs have consistently achieved an acceleration of more than 12 times, making the longest running time of all tested queries 12 times shorter. 
This can mean that some program-dataset pairs previously considered beyond the capacity of the engine may now be queried with acceptable time costs.

For the meteorological dataset case, the results resemble those of LUBM$_t$, with the least accelerated query being answered 1.58 times faster. 
Half of the queries were answered more than 3.54 times. 
One particular query was answered more than 1 million times faster.

For the iTemporal case, we achieved a more than 69,000 times performance boost on every query, some over a million times. 
As can be observed from the chart at the bottom in Figure~\ref{complubm}, the general performance boost is considerable and rather consistent across queries. 
The huge performance boost is because in this case, an \textit{idb} fact has very few relevant facts and derivations of the vast majority of \textit{idb} facts were avoided by magic programs.

\subsubsection{Scalability Experiments.} 
We conducted scalability experiments on both LUBM$_t$ and iTemporal benchmarks as they are equipped with generators allowing for constructing datasets of varying sizes. 
For LUBM$_t$, we generated datasets of 5 different sizes, $10^2- 10^6$, and used the program provided along with the benchmark; for each dataset size, we selected 30 queries. 
For iTemporal, the datasets are of sizes $10-10^5$ and the program was automatically generated; for each dataset size, we selected 270 queries. 
To rule out the impact of early termination and isolate the effect that magic sets rewriting has on the materialisation time, we only selected queries that cannot be entailed so that the finite representation is always fully computed. 
These queries were executed using both approaches, with the average execution time computed and displayed as the red ($\Pi$ for original program) and blue ($\Pi'$ for magic program) lines in Figure~\ref{scale}. 

As the chart shows, our magic programs scale better for LUBM$_t$ than the original program in terms of canonical model computation. 
For iTemporal, the first two datasets are too small that both approaches compute the results instantaneously. 
As the size of a dataset increases from $10^3$ to $10^5$, we observe a similar trend as LUBM$_t$.
\begin{figure}[t]
\begin{subfigure}{\columnwidth}
    \includegraphics[width=1\columnwidth,height=0.2\textheight]{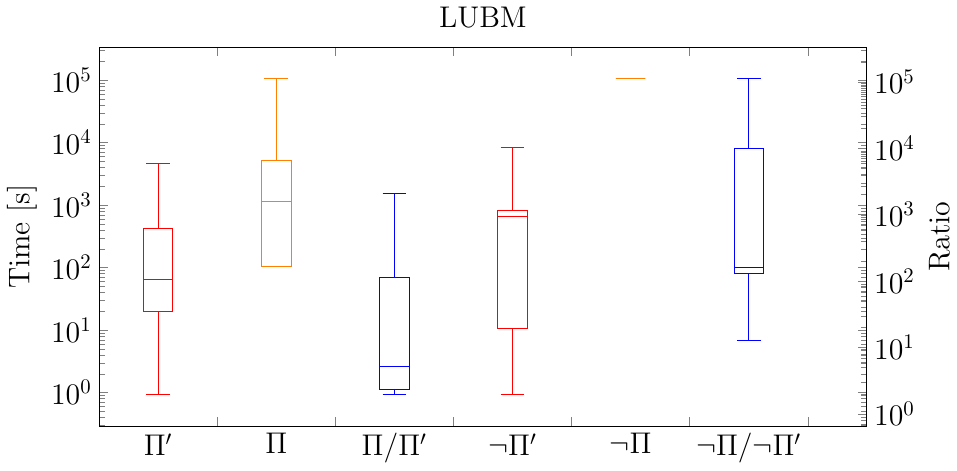}
\end{subfigure}
\begin{subfigure}{\columnwidth}
    \includegraphics[width=1\columnwidth,height=0.2\textheight]{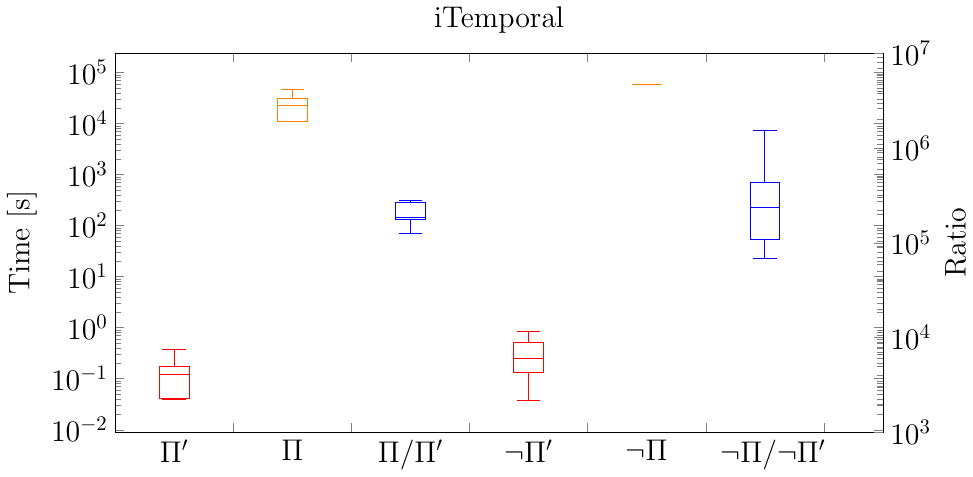}
\end{subfigure}
\begin{subfigure}{\columnwidth}
    \includegraphics[width=1\columnwidth,height=0.2\textheight]{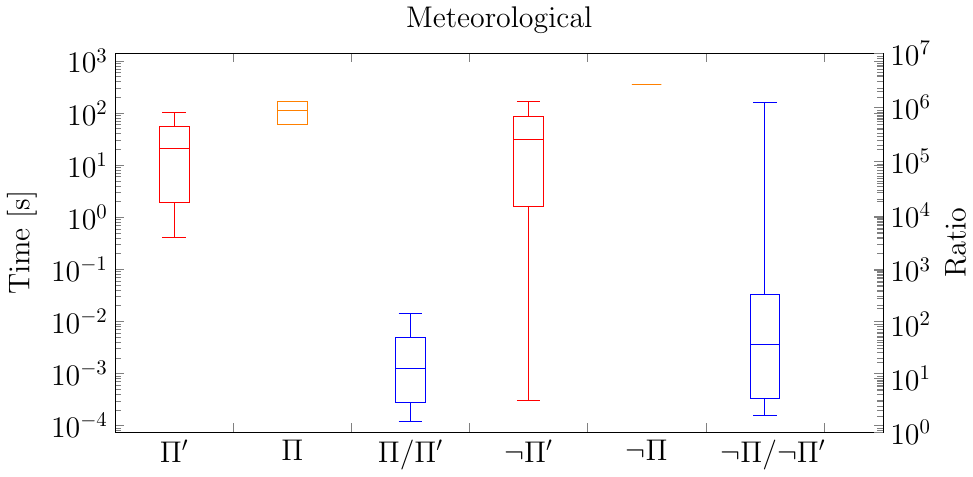}
\end{subfigure}
\caption{Comparison with baseline; ($\neg$)$\Pi'$, ($\neg$)$\Pi$, $(\neg)\Pi/(\neg)\Pi'$ refer to time consumed by (not) entailed queries on magic programs, the baseline program, and to per-query time ratios. Scales for ratios are on the right}
\label{complubm}
\end{figure}

\begin{figure}[t]
\begin{subfigure}{0.49\columnwidth}
\centering
    \includegraphics[width=\columnwidth]{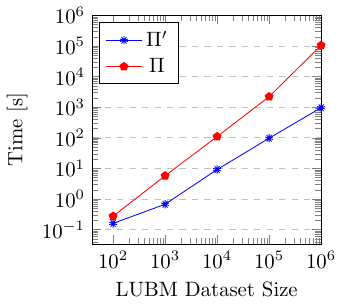}
\end{subfigure}
\begin{subfigure}{0.49\columnwidth}
\centering
    \includegraphics[width=1\columnwidth]{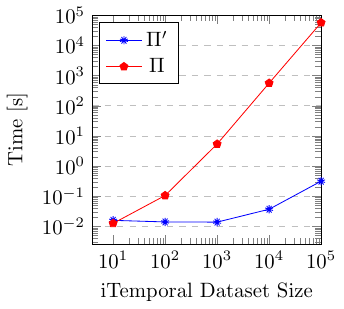}
\end{subfigure}
\caption{Scalability results, where  $\Pi'$ stands for a magic program and $\Pi$  for an original program}
\label{scale}
\end{figure}
\section{Conclusion and Future Work}
We have extended the magic set rewriting technique to the temporal setting,
which allowed us to improve  performance of query answering in DatalogMTL. 
We have obtained a goal-driven approach,  providing an alternative to the state-of-the-art materialisation-based method in DatalogMTL. 
Our new approach can be particularly useful in applications where the input dataset changes frequently or the materialisation of the entire program and dataset is not feasible due to too high time and space consumption. 
In the future, we plan to consider developing a hybrid approach that combines magic sets technique with materialisation. 
Another interesting avenue is to extend the magic sets approach to languages including such  features such as negation and aggregation.

\section{Acknowledgements}
This work was funded by NSFC grants 62206169 and 61972243.
It is also partially supported by
the EPSRC projects OASIS (EP/S032347/1), ConCuR
(EP/V050869/1) and UK FIRES (EP/S019111/1), as well as SIRIUS Centre for Scalable Data Access and Samsung Research UK.
For the purpose of Open Access, the authors have applied a CC BY public copyright licence to any 
Author Accepted Manuscript (AAM) version arising from this submission.
We thank Prof. Boris Motik for proofreading the manuscript and providing insightful comments.

\bibliography{CameraReady/LaTeX/aaai25}

\clearpage


\appendix
\section{Proofs}
Before presenting the proofs we make a few relevant definitions in addition to the ones in our paper. 

In the following text we adopt interval arithmetic. For a set $T\subseteq \mathbb{Q}\bigcup \{-\infty,\infty\}$, let $inf(T)$ represent the infimum of $T$ and $sup(T)$ the supremum of $T$. We define a function $Link: [\mathbb{R}]^2 \xrightarrow{}[\mathbb{R}]$, for two intervals $\varrho_1$ and $\varrho_2$, let $B = \varrho_1\cup\varrho_2$,
\begin{align*}
    Link(\varrho_1,\varrho_2) = &[inf(B),sup(B)] \\
    &- (B\ \cap\ \{inf(B),sup(B)\}) 
\end{align*}  
where $[inf(B),sup(B)]$ is a closed interval.

A DatalogMTL fact with a punctual interval is a punctual fact.

For a tuple of terms $\bm{t}$ and an adornment $\gamma$ of the same length, let $\bt{\bm{t}}{\gamma}$ and $\bt{\bm{t}}{\gamma}$ be the subsequences of bound terms and bound variables 
in $\bm{t}$, respectively, according to $\gamma$.

For a non-nested rule $r$ that is without nested operators, we define the \textbf{reasoning order} of the predicates in $body(r)$.

\begin{itemize}
    \item We first number the metric atoms of $r$ from left to right, starting with 1.
    \item We next generate a listing of all the predicates $p$ that appeared in $r$ as follows, note that duplicate predicates are preserved. 
    \begin{itemize}
        \item For any pair of predicates $p_1$ and $p_2$ , if $p_1$ and $p_2$ are from different metric atoms, say $M_i$ and $M_j$, and $i < j$, then in the listing $p_1$ appears on the left of $p_2$. Predicate $p_1$ appears on the right of $p_2$ when $i > j$.
        \item If $p_1$ and $p_2$ are from the same metric atom $M_i$, then $M_i$ contains MTL operator $\mathcal{S}$ or $\mathcal{U}$. Assume that the metric atom is in the form $p_1\mathcal{O}p_2$, where $\mathcal{O}$ is either $\mathcal{S}$ or $\mathcal{U}$, then $p_2$ appears on the left of $p_1$. The reverse is true for $p_1$ and $p_2$ if the metric atom is in the form $p_2\mathcal{O}p_1$.
    \end{itemize}
    \item Let the listing generated by the previous steps be $L$, we number the predicates in $L$ from left to right staring with 1, the number becoming the predicate's index in the reasoning order.
\end{itemize}

Let $\mathfrak{I}_0,\mathfrak{I}_1,...,$ and $\mathfrak{I}_0',\mathfrak{I}_1',...,$ be the transfinite sequences of interpretations defined for $\mathfrak{I}_{\mathcal{D}}$ and $\mathfrak{I}_{\mathcal{D}'}$ by the immediate consequence operators $T_\Pi$ and $T_{\Pi'}$, respectively. If for a fact $g$, there is an integer $i$ such that $\mathfrak{I}_i$ satisfies $g$ and $\mathfrak{I}_{i-1}$ does not satisfy $g$, we call this integer the birth round of $g$ regarding $(\Pi,\mathcal{D})$, and denote this integer by $\mathsf{birth}(g,\Pi,\mathcal{D})$. Facts satisfied by $\mathfrak{I}_0$ have birth round 0 regarding $(\Pi,\mathcal{D})$.

For an interpretation $\mathfrak{I}$, an interval $\varrho$, a relational atom $R'(\bm{t}')$, let the projection $\mathfrak{I} \mid_{R'(\bm{t}')}$ of $\mathfrak{I}$ over $R'(\bm{t}')$ be the interpretation that coincides with $\mathfrak{I}$ on relational atom $R'(\bm{t}')$ on $(-\infty,+\infty)$ and makes all other relational atoms false on $(-\infty,+\infty)$, and the projection $\mathfrak{I} \mid_{\varrho}$ of $\mathfrak{I}$ over $\varrho$ be the interpretation that coincides with $\mathfrak{I}$ on $\varrho$ and makes all relational atoms false outside $\varrho$. Let the combination of projections be so that $\mathfrak{I} \mid_{R'(\bm{t}'),\varrho} = (\mathfrak{I} \mid_{R'(\bm{t}')}) \mid_{\varrho}$. It is clear from this definition that interpretation $\mathfrak{I}$ is a shift of another interpretation $\mathfrak{I}'$ if and only if for any relational atom $R'(\bm{t}')$, $\mathfrak{I} \mid_{R'(\bm{t}')}$ is a shift of $\mathfrak{I}' \mid_{R'(\bm{t}')} $.We say that an interpretation $\mathfrak{I}'$ is a shift of $\mathfrak{I}$ if there is a rational number $q$ such
that $\mathfrak{I} \models M@\varrho$ if and only if $I' \models M@(\varrho + q)$, for each fact $M@\varrho$. It is clear from this definition that interpretation $\mathfrak{I}$ is a shift of another interpretation $\mathfrak{I}'$ if and only if for any relational atom $R'(\bm{t}')$, $\mathfrak{I} \mid_{R'(\bm{t}')}$ is a shift of $\mathfrak{I}' \mid_{R'(\bm{t}')} $. For the definitions of a saturated interpretation, please refer to~\citet{DBLP:conf/aaai/WalegaZWG23}.

We are now ready to prove Theorem~\ref{theorem1}.

\begin{proof}

Note that when $R$ is not an \textit{idb} predicate in $\Pi$, $\Pi'$ will be empty and both canonical models satisfy the same set of facts with predicate $R$, since this set is empty when $R$ is not \textit{edb} for $\Pi$, and when $R$ is \textit{edb}, $\mathcal{D}$ and $\mathcal{D}'$ contains the same set of \textit{edb} facts. Also, for determining whether an interpretation $\mathfrak{I}$ satisfies some relational atoms at each time point $t\in\varrho'$ for a time interval $\varrho'$, considering punctual facts alone is enough. Therefore, we only need to consider queries with \textit{idb} predicates in $\Pi$ and punctual intervals. Next, we prove the equivalence:
\begin{itemize}
    \item \{$\Leftarrow$\}:  
      We prove the if part of Theorem~\ref{theorem1} by an induction on $\mathsf{birth}(g,\Pi',\mathcal{D}')$, where fact $g$'s predicate is an \textit{idb} predicate in $\Pi$ and its interval $t_g$ is punctual. We show that if $\mathsf{birth}(g,\Pi',\mathcal{D}')=i$, then $\mathfrak{I}_i$ satisfies $g$. We capitalise on the fact that each rule $r'$ in $\Pi'$ that derives facts with predicates from the original program $\Pi$ is produced by adding a magic atom to a rule $r$ in $\Pi$.
    \begin{itemize}
        \item \textbf{Basis Step: } When $\mathsf{birth}(g,\Pi',\mathcal{D}')=0$, this case is trivial since such $g$ does not exist. 
        \item \textbf{Induction Step:} When $\mathsf{birth}(g,\Pi',\mathcal{D}')=i$ 
        and $i>0$, there must be a rule $r'\in \Pi'$, a dataset $F'$ satisfied by 
        $\mathfrak{I}_{i-1}'$, a time point $t$ and a substitution $\sigma$ such that 
        $\mathfrak{I}_{F'},t \vDash \sigma(body(r'))$ and for any interpretation 
        $\mathfrak{I}$ such that $\mathfrak{I},t \vDash \sigma(head(r'))$, $\mathfrak{I}$ 
        must satisfy $g$. In other words, $g$ was first derived by $r'$. Let $F$ be the set of all the facts in $F'$ whose predicates are 
        \textit{idb} predicates in $\Pi$, then by the induction hypothesis, 
        $\mathfrak{I}_{i-1}$ satisfies $F$. Let $r$ be the rule from which $r'$ is produced by lines 8--11 in Algorithm~\ref{algo:magic},
        it is clear that $\mathfrak{I}_{F},t \vDash \sigma(body(r))$ and $head(r)=head(r')$.
        Since $\mathfrak{I}_{i-1}$ satisfies $F$, 
        we have $\mathfrak{I}_{i-1},t \vDash \sigma(body(r))$, 
        which in turn means that $\mathfrak{I}_{i},t \vDash \sigma(head(r'))$, 
        whereby $\mathfrak{I}_i$ satisfies $g$.
    \end{itemize}
    \item \{$\Rightarrow$\}: Let $\gamma$ be the adornment of $R$ according to $\bm{t}$. It is obvious that $\forall \sigma$, $\bt{\sigma(\bm{t})}{\gamma)=\bt{\bm{t}}{\gamma}$, since $bt(\bm{t},\gamma}$ is a tuple of constants. Therefore, we only need to prove the following claim:
    \begin{claim}\label{essenceclaim}
        For any fact $g = S(\bm{t}_g)@t_g$, where $t_g$ is a time point and $S$ is an \textit{idb} predicate in $\Pi$, if there is a fact $g_m = m\_S^{\gamma_0}(\bt{\bm{t}_g}{\gamma_0))@t_g$ which has a birth round regarding $(\Pi',\mathcal{D}')$ for some adornment $\gamma_0$, and $g$ has a birth round regarding $(\Pi,\mathcal{D})$, then $g$ also has a birth round regarding $(\Pi',\mathcal{D}'}$.
    \end{claim}
    Because for any punctual fact $q'=R(\bm{t}')@t_{q'}$ satisfied by $\mathfrak{C}_{\Pi,\mathcal{D}}$ such that $t_{q'}\in\varrho$ and there is a substitution $\sigma$ for which $\sigma(\bm{t})=\bm{t}'$, we have already made sure that $\mathfrak{I}_0'$ satisfy its corresponding $q'_m$ mentioned in Claim~\ref{essenceclaim} by line 3 in Algorithm~\ref{algo:magic}, whereby if Claim~\ref{essenceclaim} stands, $q'$ is satisfied by $\mathfrak{C}_{\Pi',\mathcal{D}'}$. We prove Claim~\ref{essenceclaim} by an induction on $\mathsf{birth}(g,\Pi,\mathcal{D})$.
    \begin{itemize}
        \item \textbf{Basis Step: } when $\mathsf{birth}(g,\Pi,\mathcal{D})=1$,  there must be a rule $r\in \Pi$, a dataset $F$ satisfied by $\mathfrak{I}_\mathcal{D}$, a time point $t$ and a substitution $\sigma$ such that $\mathfrak{I}_{F},t \vDash \sigma(body(r))$ and for any interpretation $\mathfrak{I}$ such that $\mathfrak{I},t \vDash \sigma(head(r))$, $\mathfrak{I}$ must satisfy $g$. 
        Since $\mathcal{D}\subset \mathcal{D}'$, $\mathfrak{I}_{\mathcal{D}'}$ also satisfies $F$. 
        Let $r'$ be the rule produced by adding an atom containing $m\_S^{\gamma_0}$ to $body(r)$ in line 10, Algorithm~\ref{algo:magic}, then $\mathfrak{I}_{F\cup\{g_m\}},t\vDash body(r')$ and $head(r')=head(r)$, whereby since $\mathfrak{I}_{0}'=\mathfrak{I}_{\mathcal{D}}'$ satisfies $F\cup\{g_m\}$, we have $\mathfrak{I}_1',t \vDash head(r)$, and thus $\mathfrak{I}_1'$ satisfies $g$.
        \item \textbf{Induction Step: } when $\mathsf{birth}(g,\Pi,\mathcal{D})=n$ and $n>1$, there must be a rule $r\in \Pi$, a dataset $F'$ satisfied by $\mathfrak{I}_{n-1}$, a time point $t'$ and a substitution $\sigma$ such that $\mathfrak{I}_{F'},t \vDash \sigma(body(r))$ and for any interpretation $\mathfrak{I}$ such that $\mathfrak{I},t \vDash \sigma(head(r))$, $\mathfrak{I}$ must satisfy $g$. Let $r$ be a follows
        \begin{align*}
            \boxplus_\varrho S(\bm{t}_0) \xleftarrow{}\ M_1,M_2,...,M_l.
        \end{align*}
        We allow $\varrho$ to contain non-negative rationals while $inf(\varrho)*sup(\varrho)\geq 0$, 
        whereby we don't need to consider $\boxminus$ in the rule head. Let the predicates in $body(r)$ be $R_1,R_2,...,R_k$ according to their reasoning order, 
        with the tuples associated with them being $\bm{t}_1,\bm{t}_2,...,\bm{t}_k$ in $body(r)$. 
        Since $\mathfrak{I}_{F},t \vDash \sigma(body(r))$, $F'$ must contain a subset $F = \{f_j\ |\  f_j = R_j(\bm{t}_j')@\varrho_{f_j},j\in \{1,2,...,k\}\}$, and for $F$, $\sigma$ and the time point $t'$ who must be in $t_g-\varrho$, these conditions are met:
        \begin{itemize}
            \item Each fact $f_j$ either has an \textit{edb} predicate or has $\mathsf{birth}(f_j,\Pi,\mathcal{D})<\mathsf{birth}(g,\Pi,\mathcal{D})$.
            \item If $R_j$ belongs to a metric atom $\boxplus_{\varrho_{m_h}} R_j(\bm{t}_j)$, then $t' + \varrho_{m_h} = \varrho_{f_j}$. $t' - \varrho_{m_h} = \varrho_{f_j}$ should hold instead if the operator is $\boxminus$.
            \item If $R_j$ belongs to a metric atom $\diamondplus_{\varrho_{m_h}} R_j(\bm{t}_j)$, then $\varrho_{f_j}$ is punctual and $\varrho_{f_j} \in t'+\varrho_{m_h}$. $\varrho_{f_j} \in t' - \varrho_{m_h}$ should hold instead if the operator is $\diamondminus$.
            \item If $R_j$ and $R_{j-1}$ belongs to a metric atom $R_j(\bm{t_j}) \mathcal{S}_{\varrho_{m_h}} R_{j-1}(\bm{t_{j-1}})$, then $\varrho_{f_{j-1}}$ is punctual and $\varrho_{f_{j-1}} \in t' - \varrho_{m_h}$ and $\varrho_{f_j} = Link(\varrho_{f_{j-1}},\{t'\})$. $\varrho_{f_{j-1}} \in t' + \varrho_{m_h}$ should hold instead if the operator is $\mathcal{U}$.
            \item $\sigma(\bm{t}_0)=\bm{t}_g$, $\sigma(\bt{\bm{t}_0}{\gamma_0))=\bt{\bm{t}_g}{\gamma_0}$ and $\sigma(\bm{t}_j} = \bm{t}_j',\forall j \in \{1,2,...,k\}$
        \end{itemize}
        Since $g_m$ has a birth round regarding $(\Pi',\mathcal{D}')$ by the induction hypothesis, we only need to prove all facts in $F$ are either in \textit{edb} or has a birth round regarding $(\Pi',\mathcal{D}')$. The cases where $R_j$ is an \textit{edb} predicate are trivial, whereby we only need to consider when it's an \textit{idb} predicate. We prove this inductively on index $j$ of the facts in $F$.
        \begin{itemize}
            \item \textbf{Basis Step:} if $R_1$ is an \textit{idb} predicate, for the first metric atom $M_1$, 
            the rule $r_{1}$ generated for $R_1(\bm{t}_1)$ by lines 12--18 in Algorithm~\ref{algo:magic} contains only
            $\diamondplus_\varrho m\_S^{\gamma_0}(\bt{\bm{t}_0}{\gamma_0)}$ 
            in the body and has relational atom 
            $m\_R_1^{\gamma_1}(\bm{s}_1)$ 
            in the head for another adornment 
            $\gamma_1$, where 
            $\bm{s}_1 = \bt{\bm{t}_1}{\gamma_1}$. 
            Since 
            $\mathfrak{I}_{\{g_m\}},t' \vDash \diamondplus_\varrho m\_S^{\gamma_0}(\bt{\bm{t}_0}{\gamma_0)}$,
            we have 
            $\mathfrak{I}_{\mathsf{birth}(g_m,\Pi',\mathcal{D}')}',t' \vDash \sigma(body(r_1))$ 
            and 
            $\mathfrak{I}_{\mathsf{birth}(g_m,\Pi',\mathcal{D}')+1}',t' \vDash \sigma(head(r_1))$. Because 
            $\mathfrak{I}_{\mathsf{birth}(g_m,\Pi',\mathcal{D}')+1}',t' \vDash \sigma(head(r_1))$, interpretation 
            $\mathfrak{I}_{\mathsf{birth}(g_m,\Pi',\mathcal{D}')+1}'$ 
            must satisfy a fact 
            $g_1=m\_R_1^{\gamma_1}(\bm{t}_{g_1})@\varrho_{g_1}$ such that $\sigma(\bm{s}_1) = \bm{t}_{g_1}$ 
            and according to Algorithm~\ref{algo:MagicHeadAtoms} which determines $head(r_1)$,
            \begin{itemize}
                \item $\varrho_{g_1} = t' + \varrho_{m_1}$, 
                if $M_1$ is 
                $\boxplus_{\varrho_{m_1}} R_1(\bm{t}_1)$, $\diamondplus_{\varrho_{m_1}} R_1(\bm{t}_1)$ or $R_2(\bm{t}_2)\ \mathcal{U}_{\varrho_{m_1}}R_1(\bm{t}_1)$. In this case, since $\varrho_{f_1} = t' + \varrho_{m_1}$, we have $\varrho_{f_1} \subseteq \varrho_{g_1}$
                \item $\varrho_{g_1} = t' - \varrho_{m_1}$, 
                if $M_1$ is 
                $\boxminus_{\varrho_{m_1}} R_1(\bm{t}_1)$, $\diamondminus_{\varrho_{m_1}} R_1(\bm{t}_1)$ 
                or $R_2(\bm{t}_2)\ \mathcal{S}_{\varrho_{m_1}}R_1(\bm{t}_1)$. 
                Similarly, we can deduce $\varrho_{f_1} \subseteq \varrho_{g_1}$.
            \end{itemize}
            Because $\bm{t}_{g_1} = \sigma(\bm{s}_1)= \sigma(\bt{\bm{t}_1}{\gamma_1))= \bt{\bm{t}_1'}{\gamma_1}$ and $\varrho_{f_1} \subseteq \varrho_{g_1}$, $\mathfrak{I}_{\mathsf{birth}(g_m,\Pi',\mathcal{D}'}+1}'$ 
            satisfies $m\_R_1^{\gamma_1}(\bt{\bm{t}_1'}{\gamma_1)}@\varrho_{f_1}$,
            and we have that for each fact $f_1'=R_1(\bm{t}_1')@t_1'$ 
            with $t_1'\in \varrho_{f_1}$, there is a fact 
            $f_{m_1}'=m\_R_1^{\gamma_1}(\bt{\bm{t}_1'}{\gamma_1)}@t_1'$, 
            for which $\mathsf{birth}(f_{m_1}',\Pi',\mathcal{D}')\leq \mathsf{birth}(g_m,\Pi',\mathcal{D}')+1$. 
            Since $\mathsf{birth}(f_1,\Pi,\mathcal{D})<n$, we have $\mathsf{birth}(f_1',\Pi,\mathcal{D})<n$ and by the induction hypothesis we have there is a birth round for $f_1'$ regarding $(\Pi',\mathcal{D}')$ 
            and in turn, $\mathsf{birth}(f_1,\Pi',\mathcal{D}')$ is the largest $\mathsf{birth}(f_1',\Pi',\mathcal{D}')$ for any $f_1'$.
            \item \textbf{Induction Step:} if $R_j$ is an \textit{idb} predicate, Let $max_{j-1}$ be the largest integer among the birth rounds of 
            $g_m,f_1,f_2,...,f_{j-1}$ regarding $(\Pi',\mathcal{D}')$. 
            Similar to the base case, we have the rule $r_j$ generated by lines 12--18 in Algorithm~\ref{algo:magic} for $R_j(\bm{t}_j)$ contains $\diamondplus_\varrho m\_S^{\gamma_0}(\bt{\bm{t}_0}{\gamma_0)}$ plus 
            $M_1,M_2,...,M_{h-1}$ in the body, where $M_h$ contains $R_j(\bm{t}_j)$, 
            and $r_j$ contains $m\_R_j^{\gamma_j}(\bm{s}_j)$ in the head, 
            where $\bm{s}_j = \bt{\bm{t}_j}{\gamma_j}$. 
            Let $F_{j}=\{f_o | o\in\{1,2,...,j\}\}$, 
            then $\mathfrak{I}_{F_{j-1}\cup\{g_m\}}',t' \vDash \sigma(body(r_j))$,
            whereby we have that $\mathfrak{I}_{max_{j-1}}',t'\vDash \sigma(body(r_j))$,
            and $\mathfrak{I}_{max_{j-1}+1}',t'\vDash \sigma(head(r_j))$. Because $\mathfrak{I}_{max_{j-1}+1},t'\vDash head(r_j)$, interpretation $\mathfrak{I}_{max_{j-1}+1}$ must satisfy a fact $g_j=m\_R_j^{\gamma_j}(\bm{t}_{g_j})@\varrho_{g_j}$, such that $\sigma(\bm{s}_j)=\bm{t}_{g_j}$ and according to Algorithm~\ref{algo:MagicHeadAtoms} which determines $head(r_j)$,

            \begin{itemize}
                \item $\varrho_{g_j} = t' + \varrho_{m_h}$, 
                if $R_j$ is in
                $M_h = \boxplus_{\varrho_{m_h}} R_j(\bm{t}_j)$,
                $\diamondplus_{\varrho_{m_h}} R_j(\bm{t}_j)$ or $R_{j+1}(\bm{t}_{j+1})\mathcal{U}_{\varrho_{m_h}}R_j(\bm{{t}_j})$. 
                In this case, since 
                $\varrho_{f_j} = t' + \varrho_{m_h}$, 
                we have $\varrho_{f_j} \subseteq \varrho_{g_j}$
                \item $\varrho_{g_j} = t' - \varrho_{m_h}$, 
                if $R_j$ is in
                $M_h = \boxminus_{\varrho_{m_h}} R_j(\bm{t}_j)$, $\diamondminus_{\varrho_{m_h}} R_j(\bm{t}_j)$ 
                or $R_{j+1}(\bm{t}_{j+1})\mathcal{S}_{\varrho_{m_h}}R_j(\bm{t}_{j})$. 
                Similarly, we can deduce $\varrho_{f_j} \subseteq \varrho_{g_j}$.
                \item $\varrho_{g_j} = Link(\{t'\},t' - \varrho_{m_h})$, 
                if $R_j$ is in 
                $M_h = R_j(\bm{t}_{j})\mathcal{S}_{\varrho_{m_h}}R_{j-1}(\bm{t}_{j-1})$.
                In this case, 
                $\varrho_{f_j} = Link(\{t'\},\varrho_{f_{j-1}}) \subseteq Link(\{t'\},t' - \varrho_{m_h}) =  \varrho_{g_j}$. 
                Note this is because $\varrho_{f_{j-1}} \in t' - \varrho_{m_h}$.
                \item $\varrho_{g_j} = Link(\{t'\},t' + \varrho_{m_h})$, 
                if $R_j$ is in 
                $M_h = R_j(\bm{t}_{j})\mathcal{S}_{\varrho_{m_h}}R_{j-1}(\bm{t}_{j-1})$. 
                We again derive $\varrho_{f_j} \subseteq \varrho_{g_j}$.
            \end{itemize}
            Analogous to the base case, since $\bm{t}_{g_j}=\sigma(\bm{s}_j) = \sigma(\bt{\bm{t}_j}{\gamma_j)) = bt(\bm{t}_j',\gamma_j}$, 
            we have that $\mathfrak{I}_{max_{j-1}+1}'$ satisfies $m\_R_j^{\gamma_j}(\bt{\bm{t}_j'}{\gamma_j)}@\varrho_{f_j}$,
            and that for each $f_j'=R_j(\bm{t}_j')@t_j'$ with $t_j'\in\varrho_{f_j}$, there is a fact $f_{m_j}'$ for $f_j'$ such that 
            $f_{m_j}'=m\_R_j^{\gamma_j}(\bt{\bm{t}_j'}{\gamma_j)}@t_j'$ 
            and $\mathsf{birth}(f_{m_j}',\Pi',\mathcal{D}')\leq max_{j-1}+1$. 
            Since $\mathsf{birth}(f_j,\Pi,\mathcal{D})<n$, 
            we have $\mathsf{birth}(f_j',\Pi,\mathcal{D})<n$, 
            and by the induction hypothesis we have there is a birth round for $f_j'$ regarding 
            $(\Pi',\mathcal{D}')$ and in turn, $\mathsf{birth}(f_j,\Pi',\mathcal{D}')$ is the largest $\mathsf{birth}(f_j',\Pi',\mathcal{D}')$ for any $f_j'$.
            
        \end{itemize}
    \end{itemize}
\end{itemize} 
\end{proof}

Next, we deal with Theorem~\ref{migrationProp}

First, we propose a more general theorem.

Let $add(\mathfrak{I},R(\bm{t})@\varrho)$ be the least interpretation that contains $\mathfrak{I}$ and entails $R(\bm{t})@\varrho$. 
Let function $toRule(R(\bm{t})@(-\infty,+\infty))$ mapping a fact with interval $(-\infty,+\infty)$ to a rule $r$ be such that $r$ is
\begin{align*}
    R(\bm{t}) \gets \top.
\end{align*}
For a dataset $\mathcal{D}$, let $unbounded(\mathcal{D})$ be the set of facts in $\mathcal{D}$ that have interval $(-\infty,+\infty)$ and let a funtion $swap$ mapping a dataset to a program be such that
\begin{align*}
    swap(\mathcal{D}) = \{toRule(g) \mid g \in unbounded(\mathcal{D})\}
\end{align*}

\begin{theorem}\label{adapttheo}
    For bounded a DatalogMTL program-dataset pair $(\Pi,\mathcal{D})$ and a set $E$ of facts with interval $(-\infty,+\infty)$ , with $R$ an \textit{idb} predicate in $\Pi$, let $\Pi'=swap(E)\cup \Pi$, $\mathcal{D}'=\mathcal{D}\cup E$, $\mathfrak{I}_0,\mathfrak{I}_1,...,$ and $\mathfrak{I}_0',\mathfrak{I}_1',...,$ be the transfinite sequences of interpretations defined for $\mathfrak{I}_{\mathcal{D}}$ and $\mathfrak{I}_\mathcal{D'}$ by the immediate consequence operators $T_\Pi$ and $T_{\Pi'}$, respectively. Then, for each $i=0,1,2\dots$, we have $\mathfrak{I}_i' = \mathfrak{I}_{i+1}$.
\end{theorem}
\begin{proof}
    We prove Theorem~\ref{adapttheo} by an induction on $i$ for $\mathfrak{I}_i' = \mathfrak{I}_{i+1}$.
    \begin{itemize}
        \item \textbf{Basis Step}: Since in the first application of $T_{\Pi'}$, only rules in $swap(E)$ have their body satisfies, this makes $\mathfrak{I}_1$ satisfy every fact in $E$, and we have $\mathfrak{I}_1$ satisfies every fact that $ \mathfrak{I}_0'$ does and vice versa, whereby $\mathfrak{I}_1= \mathfrak{I}_0'$
        \item \textbf{Induction Step}: Since $\mathfrak{I}_1$ already satisfies $E$, for $i\geq 1$, no fact that is satisfied by $\mathfrak{I}_{i+1}$ but not $\mathfrak{I}_{i}$ can have the same relational atom as any fact in $E$. This in turn means that $\mathfrak{I}_{i+1} = T_{\Pi'}(\mathfrak{I}_{i}) = T_{\Pi}(\mathfrak{I}_{i})$, for any $i\geq 1$. By the induction hypothesis we have $\mathfrak{I}_j' = \mathfrak{I}_{j+1}$ for $j\geq 0$, and $\mathfrak{I}_{j+2} = T_{\Pi'}(\mathfrak{I}_{j+1}) =T_{\Pi}(\mathfrak{I}_{j+1}) = T_{\Pi}(\mathfrak{I}_{j}') = \mathfrak{I}_{j+1}'$
    \end{itemize}
\end{proof}

Next, we use Theorem~\ref{adapttheo} to prove Theorem~\ref{migrationProp}.

\begin{proof}
    Observe that the program $\Pi'$ constructed by our Algorithm~\ref{algo:magic} is bounded and the only  unbounded fact in $\mathcal{D}'$ is $q_m' = m\_Q^{\gamma_0}(\bt{\bm{t}}{\gamma_0))@(-\infty, +\infty}$.
    Let $\Pi''=\Pi'\cup swap(\{q_m'\})$, and let 
    $\mathfrak{I}_0,\mathfrak{I}_1,...,$ and $\mathfrak{I}_0',\mathfrak{I}_1',...,$ be the transfinite sequences of interpretations defined for $\mathfrak{I}_{\mathcal{D}}$ and $\mathfrak{I}_\mathcal{D'}$ by the immediate consequence operators $T_{\Pi''}$ and $T_{\Pi'}$, respectively. Then by Theorem~\ref{adapttheo} we have that for each $i=0,1,2\dots$, $\mathfrak{I}_i' = \mathfrak{I}_{i+1}$. Because $(\Pi'', \mathcal{D})$ are bounded, there is some $k\geq 0$ for which $\mathfrak{I}_{k+1}$ is a saturated interpretation, which means $\mathfrak{I}_k'$ is also a saturated interpretation, considering that the rule $ m\_Q^{\gamma_0}(\bt{\bm{t}}{\gamma_0))\gets \top$ in $\Pi''$ is satisfied by every  interpretation $\mathfrak{I}_i$ with $i\geq 1$. This guarantees that application of Algorithm 1 by~\citet{DBLP:conf/aaai/WalegaZWG23} to $(\Pi',\mathcal{D}',q}$ terminates after $k$ application of $T_{\Pi'}$.
\end{proof}
\end{document}